%% file: main.tex
\documentclass[11pt]{article}
\usepackage[letterpaper,margin=1in]{geometry}
\usepackage{amsmath,amssymb,amsthm,mathtools}
\usepackage{lmodern,microtype,booktabs,enumitem,graphicx,float}
\usepackage[round,authoryear]{natbib}
\usepackage{tikz}
\usetikzlibrary{arrows.meta,positioning}
\usepackage[colorlinks=true,linkcolor=blue!45!black,citecolor=blue!45!black,urlcolor=blue!45!black]{hyperref}
\hypersetup{pdftitle={Exponential Hardness of Off-Policy Evaluation under History-Dependent Logging},pdfauthor={Pranaya Jajoo}}
\setlist{itemsep=3pt,topsep=5pt}
\numberwithin{equation}{section}
\newtheorem{theorem}{Theorem}[section]
\newtheorem{lemma}[theorem]{Lemma}
\newtheorem{proposition}[theorem]{Proposition}
\newtheorem{corollary}[theorem]{Corollary}
\theoremstyle{definition}
\newtheorem{definition}[theorem]{Definition}
\theoremstyle{remark}
\newtheorem{remark}[theorem]{Remark}
\newcommand{\E}{\mathbb E}
\newcommand{\Pp}{\mathbb P}

\newcommand{\one}{\mathbf 1}
\newcommand{\diag}{\operatorname{diag}}
\newcommand{\KL}{D_{\mathrm{KL}}}
\newcommand{\kl}{\operatorname{kl}}
\newcommand{\TV}{\operatorname{TV}}
\newcommand{\Ber}{\operatorname{Ber}}
\newcommand{\Var}{\operatorname{Var}}
\newcommand{\Aff}{\operatorname{Aff}}
\newcommand{\hold}{\mathsf{k}}
\newcommand{\resetzero}{\mathsf{z}}
\newcommand{\resetone}{\mathsf{o}}
\newcommand{\never}{\mathsf{I}}
\newcommand{\calE}{\mathcal E}
\newcommand{\calM}{\mathcal M}
\newcommand{\pb}{\pi_b}
\newcommand{\pe}{\pi_e}
\newcommand{\J}{J}
\title{\vspace{-1.1em}\textbf{Exponential Hardness of Off-Policy Evaluation\\
under History-Dependent Logging}}
\author{Pranaya Jajoo\\[0.3em]\normalsize University of Alberta}
\date{}
\begin{document}
\maketitle
\vspace{-2.5em}
\begin{abstract}
Can a logged dataset visit every hidden state frequently and still be
exponentially uninformative about a target policy's value? We show that
it can when the logger depends on history. For every horizon $H\ge3$, we
construct two POMDPs with at most two latent states per stage, three
actions, and a common logger with three memory states. Action coverage,
belief coverage, and two behavior-marginal outcome-revealing conditions
all have constants independent of $H$. Nevertheless, evaluating a known
deterministic target policy to accuracy $1/8$ requires
$\Theta((3/2)^H\log(1/\delta))$ logged episodes at confidence $1-\delta$,
for $0<\delta\le1/4$, even when both candidate models are known.
The mechanism is simple: a reset erases the unknown transition that
determines the target value. We characterize the resulting statistical
experiment exactly and obtain a matching optimal estimator. A directed
two-lane gridworld realizes the construction, and trajectory simulations
agree with its finite-sample prediction. The result establishes
intractability for the history-dependent-logging, model-based case
posed by \citet{zhang2025}, under their behavior-marginal definition
of revealing.
\end{abstract}

\section{Introduction}
Off-policy evaluation (OPE) estimates a target policy's return from
trajectories collected by another policy, the logger. In a partially
observable Markov decision process (POMDP), the learner must account
for both hidden state and the difference between the two policies.
Future-dependent value functions exploit latent-state structure by
using future observations and actions as proxies for the current state
\citep{uehara2023,zhang2024}. A central question is whether coverage
and observability conditions on this small state space suffice when
the logger depends on history.

We show that strong versions of both conditions can hold while the
data are exponentially uninformative about the target value. The
construction separates predicting a \emph{current state} from identifying
an earlier \emph{transition parameter}. An agent moves through two
visually identical corridors. Gates reset its hidden lane, making
the current state predictable from the logger's actions but erasing
the unknown initial transition. The target never resets, so that
erased information determines its return.

\paragraph{Contributions.}
We construct two POMDPs with constant action coverage, belief coverage,
and two behavior-marginal outcome-revealing constants, yet prove an
exponential sample lower bound for every estimator given the complete
candidate models and target policy. An exact statistical reduction
yields the optimal estimator and a matching upper bound. Gridworld
simulations reproduce its finite-sample error, and an analysis of
logger memory explains why state-based factorizations fail.

\paragraph{Relation to prior work and scope.}
\citet{zhang2025} leave open the model-based, history-dependent-behavior,
multi-step-revealing case in their Table~1 and conjecture intractability.
Their model-free lower bound restricts target-policy queries. Our
construction instead varies the environment and provides the same
completely known target policy, with unrestricted queries and computation.
It already works for a memoryless target. Revealing throughout this
paper uses the logger's future law averaged over histories conditional
on the physical state; stronger history-conditional observability is
a different assumption. The lower bound concerns fixed offline data,
not active interaction with the environment.

\section{A two-lane construction}
\label{sec:construction}
Every nonterminal action advances one column of a directed two-lane
grid. Continue preserves the lane; an upper or lower gate sets it to
zero or one. The corridors look identical to the learner, and only
the final noisy reward depends on the lane (Figure~\ref{fig:mechanism}).

\subsection{Two candidate environments}
Let \(\mathcal A=\{\hold,\resetzero,\resetone\}\), interpreted as continue (hold the lane),
upper gate (reset to zero), and lower gate (reset to one). The initial state is the singleton
\(\star\); every later state is a bit. In model \(M_\theta\),
\(\theta\in\{0,1\}\), the first transition is deterministic:
\begin{equation}
 T_{1,\theta}(\star,\hold)=\theta,\qquad
 T_{1,\theta}(\star,\resetzero)=0,\qquad
 T_{1,\theta}(\star,\resetone)=1 .
 \label{eq:first}
\end{equation}
Here and below a deterministic transition is denoted by its next state.
For \(2\le h<H\), both models have the same transitions
\begin{equation}
 T_h(s,\hold)=s,\qquad T_h(s,\resetzero)=0,\qquad
 T_h(s,\resetone)=1 .
 \label{eq:later}
\end{equation}
Every observation before \(H\) is a fixed time-tagged blank symbol. The
terminal observation \(Y\in\{0,1\}\) follows the common channel
\begin{equation}
 \Pp(Y=1\mid s_H=0)=\frac14,\qquad
 \Pp(Y=1\mid s_H=1)=\frac34 .
 \label{eq:emission}
\end{equation}
The known reward is zero before \(H\) and equals \(Y\) at \(H\). The
candidates differ only in \eqref{eq:first}, and have identical emissions
and rewards. The final action affects neither the reward nor any observation.

\begin{figure}[!htbp]
\centering
\includegraphics[width=\linewidth]{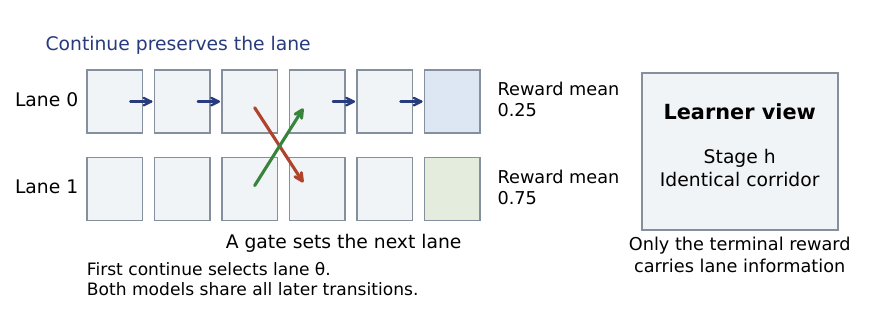}
\caption{The two-lane gridworld. The learner sees identical corridors;
this explanatory rendering exposes the lane. Models differ only in
the first continue transition. A gate resets the lane and erases
that difference.}
\label{fig:mechanism}
\end{figure}

\subsection{A common logger with three memory labels}
At stage 1, the logger selects each action with probability \(1/3\).
For \(2\le h\le H\), define a label \(m_h\in\{\never,0,1\}\) from the
previous actions. The label is \(\never\) if no reset has occurred; otherwise
it is the value of the most recent reset. For \(h<H\), use the following action rule:
\begin{center}
\begin{tabular}{lccc}
\toprule
Logger memory & Hold \(\hold\) & Reset zero \(\resetzero\) & Reset one \(\resetone\)\\
\midrule
Never reset, \(\never\) & \(2/3\) & \(1/6\) & \(1/6\)\\
Last reset zero, \(0\) & \(1/6\) & \(2/3\) & \(1/6\)\\
Last reset one, \(1\) & \(1/6\) & \(1/6\) & \(2/3\)\\
\bottomrule
\end{tabular}
\end{center}
At stage \(H\), choose an independent uniform action. This rule ignores
observations and is defined at every history. It has no dependence on
\(\theta\) or on the unobserved state. Thus the same known logger is valid
in both models and \(C_A=6\).

The target always chooses \(\hold\). Its bit remains \(\theta\) from stage
2 onward, so \(\J_{M_\theta}(\pe)=1/4+\theta/2\).

\paragraph{Why the target needs rare histories.}
After the first gate, the last gate determines the lane in either
model. The remaining observations then have the same distribution
under both candidates. Thus a logged episode is informative about
the unknown first transition only if it continues at every step.
That event has probability
\begin{equation}
 p_H=\frac13\left(\frac23\right)^{H-2}.
 \label{eq:intuitivep}
\end{equation}
The target follows this informative path in every episode.

\section{Setup and coverage conditions}
\label{sec:setup}
We consider a layered horizon-$H$ POMDP with finite state sets
$\mathcal S_h$, observation sets $\mathcal O_h$, action set $\mathcal A$,
transition kernels $T_h$, and emission kernels $O_h$.
At stage \(h\), a hidden state \(s_h\in\mathcal S_h\) emits
\(o_h\sim O_h(\cdot\mid s_h)\), the learner receives the known reward
\(R(o_h)\in[0,1]\), and then chooses \(a_h\in\mathcal A\). For \(h<H\),
the next state follows \(T_h(\cdot\mid s_h,a_h)\).
There is no transition after the final action. Alphabets are finite and may be
time-tagged. Let
\[
\tau_{h-1}=(o_1,a_1,\ldots,o_{h-1},a_{h-1}),\qquad \tau_0=\varnothing.
\]
An unconfounded policy chooses \(a_h\) according to
\(\pi_h(\cdot\mid\tau_{h-1},o_h)\); it has access only to observable history.
A memoryless policy depends only on \(h,o_h\). Write
\(\J_M(\pi)=\E_M^\pi[\sum_{h=1}^H R(o_h)]\).

The learner is given a finite candidate class \(\calM\), common known behavior
and target policies \(\pb,\pe\), and \(n\) independent complete trajectories
from \(\pb\) in an unknown \(M\in\calM\). Candidate models share their spaces
and reward function. The learner can inspect and simulate any candidate,
and can query \(\pe\) on arbitrary histories, but receives no additional
samples from the unknown environment. An estimator is
\((\epsilon,\delta)\)-accurate if
\begin{equation}
 \sup_{M\in\calM}
 \Pp_M^{\pb}\!\left(\left|\widehat J-\J_M(\pe)\right|>\epsilon\right)
 \le\delta .
 \label{eq:accuracy}
\end{equation}
The probability includes the estimator's randomization.

Action coverage prevents negligible logging probabilities. Belief
coverage asks whether histories span the physical states; outcome
revealing asks whether logged futures distinguish them. These are
properties of each candidate, not privileged information given to
the estimator.

The following matrices use the physical latent state. The belief before the
stage-\(h\) observation is
\[
b_{M,h}(\tau_{h-1})
   =\bigl(\Pp_M(s_h=s\mid\tau_{h-1})\bigr)_{s\in\mathcal S_h}.
\]
The future is
\(F_h=(o_h,a_h,\ldots,o_{H-1},a_{H-1},o_H)\).
Define the outcome matrix and state prior by
\begin{equation}
 U_{M,h}(f,s)=\Pp_M^{\pb}(F_h=f\mid s_h=s),
 \qquad p_{M,h}(s)=\Pp_M^{\pb}(s_h=s).
 \label{eq:outcome}
\end{equation}
\textbf{The conditioning matters.}
Equation~\eqref{eq:outcome} averages the logger's past memory
conditional on $s_h$; it does not fix a history. This is the
\emph{behavior-marginal} definition. State priors are positive in our
construction, and zero-probability outcome rows are omitted.

\begin{definition}[Coverage and revealing constants]
\label{def:coverage}
For \(h\in[H-1]\), let
\begin{align}
 \Sigma_{M,h}
 &=\E_M^{\pb}[b_{M,h}(\tau_{h-1})b_{M,h}(\tau_{h-1})^\top], \label{eq:beliefgram}\\
 G_{M,h}
 &=U_{M,h}^\top\diag(U_{M,h}\one)^{-1}U_{M,h}, \label{eq:uniformgram}\\
 K_{M,h}
 &=\diag(p_{M,h})U_{M,h}^\top
       \diag(U_{M,h}p_{M,h})^{-1}U_{M,h}. \label{eq:weightedgram}
\end{align}
Uniform action coverage, belief coverage, uniform-prior outcome revealing,
and prior-weighted outcome revealing hold with constants
\(C_A,C_H,C_F,\widetilde C_F\), respectively, if
\begin{equation}
 \pi_{b,h}(a\mid\tau,o)\ge C_A^{-1}
       \quad\text{at every stage and history},
 \label{eq:actioncover}
\end{equation}
\begin{equation}
 \lambda_{\min}(\Sigma_{M,h})\ge C_H^{-1},\qquad
 \|G_{M,h}^{-1}\|_1\le C_F,\qquad
 \|K_{M,h}^{-1}\|_1\le\widetilde C_F
 \label{eq:matrixcover}
\end{equation}
for every required \(M,h\). Here
\(\|B\|_1=\max_j\sum_i|B_{ij}|\) is the induced matrix 1-norm.
\end{definition}

Equations \eqref{eq:beliefgram}--\eqref{eq:uniformgram} match
Assumptions C--D of \citet{zhang2025};
\eqref{eq:weightedgram} is the alternative in their Appendix G.2,
Assumption F. We use their stage range \(h<H\).
The displayed bounds supplied to the learner are common to both
candidates. At $h=H$, uniform-prior revealing also holds with constant
four (Remark~\ref{rem:terminal}).

Write \([m]=\{1,\ldots,m\}\) for a positive integer \(m\).
All logarithms are natural. For \(x,y\in(0,1)\), let
\(\kl(x,y)=x\log(x/y)+(1-x)\log((1-x)/(1-y))\).
For discrete laws \(P,Q\), write
\(\Aff(P,Q)=\sum_x\sqrt{P(x)Q(x)}\) and
\(\TV(P,Q)=\tfrac12\sum_x|P(x)-Q(x)|\).

\section{Main result}
\label{sec:main}
\begin{theorem}[Exponential offline evaluation complexity]
\label{thm:main}
For every integer \(H\ge3\), there is a candidate class
\(\calM_H=\{M_0,M_1\}\) and common known policies \(\pb,\pe\) such that:
\begin{enumerate}[label=(\roman*)]
 \item \(|\mathcal S_1|=1\), \(|\mathcal S_h|=2\) for \(h\ge2\),
       and \(|\mathcal A|=3\). Each nonterminal observation alphabet is a
       singleton and \(|\mathcal O_H|=2\). Total return lies in \([0,1]\).
 \item The logger is unconfounded, uses three memory labels,
       and assigns probability at least \(1/6\) to every action.
       The target is deterministic and memoryless.
 \item Both models satisfy Definition~\ref{def:coverage} with
 \begin{equation}
 C_A=6,\qquad C_H=3,\qquad C_F=35/9,\qquad \widetilde C_F=9.
 \label{eq:constants}
 \end{equation}
 \item \(\J_{M_0}(\pe)=1/4\) and \(\J_{M_1}(\pe)=3/4\).
       Every \((1/8,\delta)\)-accurate estimator, \(0<\delta<1/2\), requires
 \begin{equation}
 n\ge
 \frac{2\kl(1-\delta,\delta)}{p_H\log3},
 \qquad p_H=\frac13\left(\frac23\right)^{H-2}.
 \label{eq:mainlower}
 \end{equation}
\end{enumerate}
\end{theorem}

The theorem permits unrestricted computation and arbitrary queries to the
supplied target policy. For fixed \(\delta<1/2\), the lower bound is exponential in \(H\), although
there are only \(2H-1\) latent states across all layers. The constant-accuracy
choice in the theorem already rules out any uniform sample guarantee
polynomial in the horizon, model-class description, inverse accuracy,
and the displayed coverage constants.

\begin{corollary}[Matching dependence on horizon and confidence]
\label{cor:rate}
Let \(N_H(\delta)\) be the smallest integer \(n\) for which a
\((1/8,\delta)\)-accurate estimator exists for the class in
Theorem~\ref{thm:main}. With \(c_\star=1-\sqrt3/2\),
\begin{equation}
 \frac{2\kl(1-\delta,\delta)}{p_H\log3}
 \le N_H(\delta)
 \le \left\lceil
    \frac{\log(1/(2\delta))}{-\log(1-c_\star p_H)}
       \right\rceil,\qquad 0<\delta<1/2 .
 \label{eq:sandwich}
\end{equation}
Consequently,
\(N_H(\delta)=\Theta((3/2)^H\log(1/\delta))\) uniformly over
\(H\ge3\) and \(0<\delta\le1/4\).
\end{corollary}

Section~\ref{sec:exact} gives an estimator attaining the upper bound.
Appendix~\ref{sec:proof} proves every coverage claim and the lower bound.

\subsection{Proof outline}
\textbf{Coverage.} The two reset-memory labels have equal probability,
so every physical state has mass at least $1/3$. In each fixed candidate,
the action history determines the lane, making its belief matrix
diagonal. A suitable next gate has conditional probability $1/6$ in
one physical state and $2/3$ in the other. This separation gives both
revealing bounds (Appendix~\ref{sec:proof}).

\textbf{Information.} Every episode containing a gate has the same
observable law under both candidates. On the remaining fraction $p_H$,
the terminal bit is $\Ber(1/4)$ or $\Ber(3/4)$, so the trajectory laws
satisfy $\KL(P_0\Vert P_1)=p_H\log(3)/2$. An accurate value estimate
distinguishes the target values $1/4$ and $3/4$. Binary testing and
independence therefore imply $n p_H\log(3)/2\ge\kl(1-\delta,\delta)$.
The familiar testing reduction is the final step; the construction
keeps all coverage constants bounded while making $p_H$ exponentially
small.

\section{The optimal estimator and its exact error}
\label{sec:exact}
There is no optimization problem hidden in this example. Discard
episodes that pass through a gate. Among the remaining episodes,
count terminal rewards equal to one and zero. Select $M_1$ if ones
are more frequent, $M_0$ if zeros are more frequent, and flip a fair
coin at a tie. Return the selected model's target value.

Encode a reset episode by an erasure $\bot$ and a no-reset episode
by its terminal bit. This compression loses no information about the
model. Complete proofs are in Appendix~\ref{app:exactproof}.

\begin{proposition}[Three-symbol reduction]
\label{prop:experiment}
Map a complete trajectory to \(W=\bot\) if a reset occurs before the
terminal observation, and to \(W=Y\) otherwise. In the coordinate order
\((\bot,0,1)\), its distributions are
\begin{equation}
 Q_0=(1-p_H,3p_H/4,p_H/4),\qquad
 Q_1=(1-p_H,p_H/4,3p_H/4).
 \label{eq:ternary}
\end{equation}
There exists a reconstruction kernel, common to both candidates, that
generates the full trajectory conditional on \(W\). Consequently the
full-data and three-symbol experiments have the same optimal testing risk.
\end{proposition}

\begin{theorem}[Finite-sample minimax error]
\label{thm:minimax}
Let \(X_i\) be \(0,-1,+1\) when \(W_i\) equals \(\bot,0,1\), respectively,
and let \(S_n=\sum_{i=1}^n X_i\).
The optimal worst-case failure probability for estimating
\(\J_{M_\theta}(\pe)\) to accuracy \(1/8\) from \(n\) trajectories is
\begin{equation}
 e_n=\Pp_0(S_n>0)+\frac12\Pp_0(S_n=0).
 \label{eq:exactrisk}
\end{equation}
Moreover,
\begin{equation}
 \frac12(1-p_H)^n
 \le e_n\le
 \frac12\left[1-(1-\sqrt3/2)p_H\right]^n .
 \label{eq:riskbounds}
\end{equation}
\end{theorem}

The likelihood ratio is $3^{S_n}$. Symmetry makes the fair-tie test
minimax, and returning the selected target value equates testing
and estimation risk. Thus the exponential cost persists for an
explicit optimal estimator, independent of representation or optimization.

\section{Gridworld experiments}
\label{sec:experiments}
We simulate the construction action by action at
$H\in\{4,8,12,16,20,24\}$, using 256 independent datasets per candidate
and 512 per budget. The estimator receives observable evidence only;
we do not sample from the reduced three-symbol law. Each gate counts
as one transition in this custom directed grid. The experiment is an
exact spatial realization of the theorem, not a conventional four-neighbor
navigation benchmark.

Budgets are rounded values $n=c/p_H$ for
$c\in\{1/8,1/4,1/2,1,2,4,8,12\}$, plus the analytically chosen
$N_H(0.1)$. Dataset prefixes are shared across budgets; replicates
and candidates use independent random streams. We measure failure
at accuracy $1/8$. The optimal test's two model-specific errors are
equal, so pooled failures estimate its worst-case risk. Intervals are
pointwise 95\% Wilson intervals over 512 datasets, not simultaneous
bands. The full protocol and supplementary results appear in
Appendix~\ref{app:experiment_details}.

\begin{figure}[!htbp]
\centering
\includegraphics[width=\linewidth]{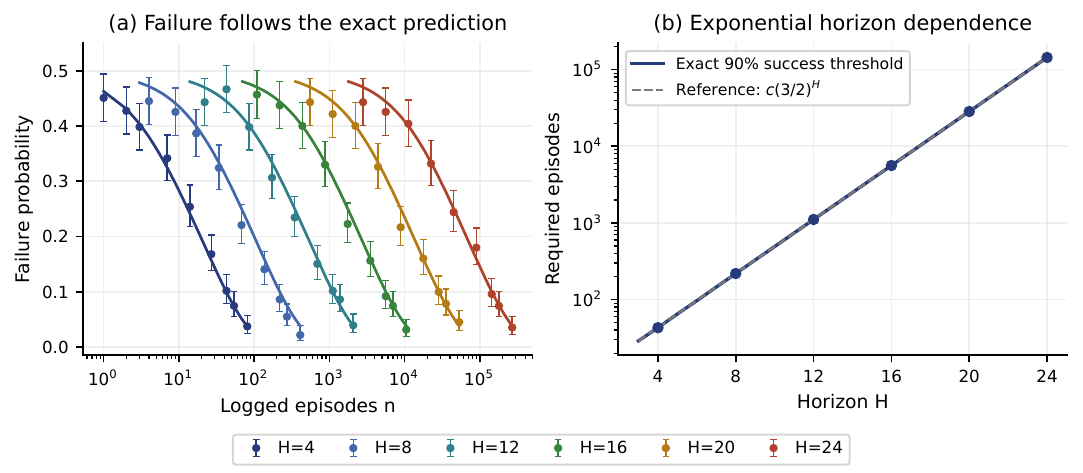}
\caption{\textbf{Optimal evaluation remains expensive.} (a) Exact risk
curves and failures from 512 simulated datasets per budget, with
pointwise 95\% Wilson intervals. (b) The exact 90\%-success threshold;
the dashed $(3/2)^H$ reference is normalized at $H=12$, not fitted
to empirical measurements.}
\label{fig:horizon}
\end{figure}

\paragraph{Results.}
The simulated errors follow the predicted horizon dependence
(Figure~\ref{fig:horizon}). At the six analytic 90\%-success thresholds,
observed failures range from 8.6\% to 10.2\%, and every interval
contains its analytic prediction (Table~\ref{tab:thresholds}).
The threshold grows from 43 episodes at $H=4$ to 143,982 at $H=24$;
it is computed from the exact risk, not fitted to simulation data.
Population diagnostics give $C_A=6$ and $C_H=3$ throughout. At $H=24$,
the actual worst-stage norms are $C_F\approx3.094$ and
$\widetilde C_F\approx4.218$, both below their proved bounds
(Figure~\ref{fig:controls}).

\paragraph{Information controls.}
We repeat the experiment at $H\in\{4,12,24\}$ after either collecting
fresh target-policy episodes or revealing the lane at stage 2 while
keeping the logger. In the first control every episode is informative;
in the second a first continue identifies the model exactly, giving
risk $\tfrac12(2/3)^n$. Their analytic 90\%-success budgets are seven
and four episodes, respectively, independent of $H$. Each episode
still costs order $H$ steps. These controls change the information
available and therefore fall outside the original lower bound.
The experiment illustrates the specified worst-case family, not
typical navigation tasks or a general ranking of learned OPE methods.

\begin{table}[!htbp]
\centering
\small
\input{threshold_table.tex}
\caption{Simulation at the analytically chosen 90\%-success budget.
Each row uses 256 datasets per candidate. Intervals concern the failure
probability at the stated budget; they are not intervals for $N_H(0.1)$.}
\label{tab:thresholds}
\end{table}

\begin{figure}[H]
\centering
\includegraphics[width=\linewidth]{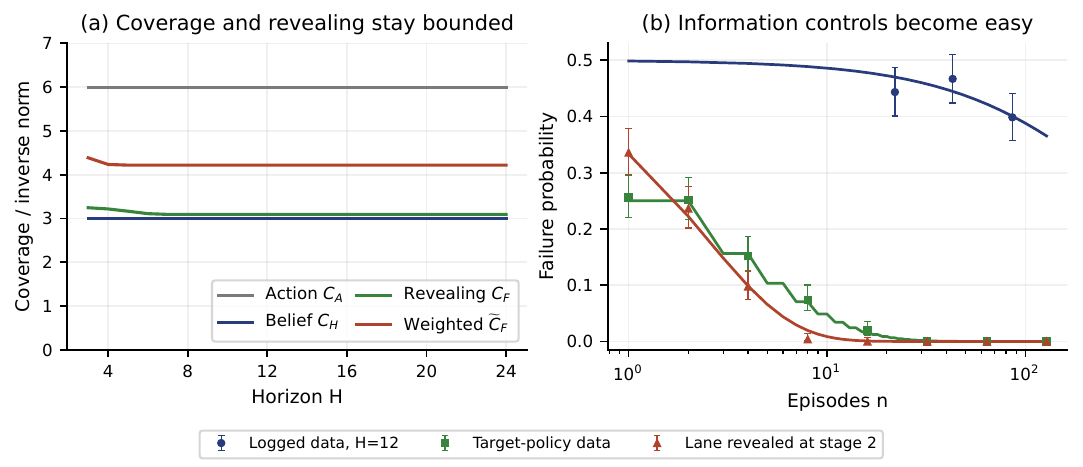}
\caption{\textbf{The obstacle is missing information about the target's
path.} (a) Population constants, maximized over both candidates and
$h<H$, stay bounded. (b) Exact curves and simulated failure probabilities
for the original logger and two information controls at $H=12$.
Every marker uses 512 datasets with a pointwise 95\% Wilson interval.
Zero observed failures do not imply zero true failure probability.}
\label{fig:controls}
\end{figure}

\section{Why current-state information is insufficient}
\label{sec:interpretation}
A history-dependent logger can behave differently after histories
that end in the same physical state. In $M_0$, a first continue and
a first upper gate both place the agent in lane zero, but the next-action
laws are $(2/3,1/6,1/6)$ and $(1/6,2/3,1/6)$. Therefore a single
state-conditioned future matrix cannot reproduce the logger's future
law after both histories:
\[
 \Pp^{\pb}(F_h=\cdot\mid\tau_{h-1})
 \ne U_{M,h}b_{M,h}(\tau_{h-1})
 \quad\text{for some histories}.
\]
The revealing assumption certifies a state--future correlation under
the logger's marginal law. It does not make the physical state
sufficient for the logger's continuation.

Augmenting the state with logger memory restores a memoryless
representation but exposes the rare never-reset state. Its belief
second-moment eigenvalue is $w_h=\tfrac13(2/3)^{h-2}$, so the coverage
cost is at least $1/w_h$. A constant number of memory labels also does
not give a uniformly short history window: arbitrarily long common
hold suffixes can retain different states or different logger laws.
These observations explain why positive results with forgetting and
additional coverage conditions, such as \citet{zhu2026}, need not
apply. Formal statements and proofs are in Appendix~\ref{app:memory}.

\section{Discussion}
Our construction separates state coverage from the information needed
to evaluate a target policy. Exponential sample complexity persists
for an optimal estimator over two known models. Positive guarantees
need additional structure tying the logger's continuation to the actual
history, through joint state coverage, forgetting, or informative
controlled futures. The theorem rules out the specified
behavior-marginal guarantees, without characterizing all tractable
POMDPs.
\clearpage
\section*{Acknowledgments}
We acknowledge the use of ChatGPT (OpenAI) in developing the theoretical
arguments and experiments and in drafting and editing the manuscript.
The proofs of the main theorem, the finite-sample minimax characterization,
and the matching sample-complexity bounds have been formalized and
verified in Lean~4. The formalization is available at
\url{https://github.com/pranayajajoo/pomdp-logging-hardness}.

\bibliographystyle{plainnat}
\bibliography{references}
\appendix
\clearpage
\section{Complete proof of the lower bound}
\label{sec:proof}
\subsection{Uniform belief coverage}
\begin{lemma}[State and logger-memory distributions]
\label{lem:belief}
For either model and every \(2\le h\le H\), put
\(w_h=\tfrac13(2/3)^{h-2}\). Then
\begin{equation}
 \Pp^{\pb}(m_h=\never)=w_h,\qquad
 \Pp^{\pb}(m_h=0)=\Pp^{\pb}(m_h=1)=\frac{1-w_h}{2}.
 \label{eq:memorymass}
\end{equation}
The physical-state belief is a coordinate vector. Up to exchanging
coordinates, its second-moment matrix is
\[
 \Sigma_{M_\theta,h}=
 \diag\left(\frac{1+w_h}{2},\frac{1-w_h}{2}\right).
\]
Hence \(C_H\le3\).
\end{lemma}
\begin{proof}
At stage 2 the three labels each have mass \(1/3\). The never-reset label
persists only after hold, with probability \(2/3\), and can never be
re-entered. The reset-label masses remain equal by symmetry of their
transition rule and the equal inflow from \(\never\). This proves
\eqref{eq:memorymass} by induction.

For a fixed model,
\begin{equation}
 s_h=\theta\quad\text{if }m_h=\never,\qquad
 s_h=m_h\quad\text{otherwise}.
 \label{eq:statefrommemory}
\end{equation}
Both quantities on the right are determined by the action prefix, so the
model-dependent belief is a coordinate vector. The state \(\theta\) has
probability \(w_h+(1-w_h)/2\); the other state has probability
\((1-w_h)/2\). Because \(w_h\le1/3\), both diagonal entries are at least
\(1/3\). At stage 1 the belief matrix is the scalar one.
\end{proof}

The model-dependent posterior in this lemma is used only to establish
coverage; it is not supplied to the learner. Unlike a lower bound whose
progress state is rarely visited, this example has constant marginal mass
on every physical state.

\subsection{Uniform-prior outcome revealing}
\begin{lemma}[Coarsening the outcome matrix]
\label{lem:coarsen}
Let \(U\) be a nonnegative matrix whose columns sum to one. Partition its
rows and form \(\bar U\) by summing rows within each part. Then
\[
 U^\top\diag(U\one)^{-1}U
 \succeq \bar U^\top\diag(\bar U\one)^{-1}\bar U ,
\]
after removing rows with zero total mass.
\end{lemma}
\begin{proof}
For any vector \(v\) and part \(B\), weighted Cauchy--Schwarz gives
\[
 \sum_{f\in B}\frac{(U(f,\cdot)v)^2}{U(f,\cdot)\one}
 \ge
 \frac{\bigl(\sum_{f\in B}U(f,\cdot)v\bigr)^2}
      {\sum_{f\in B}U(f,\cdot)\one}.
\]
Sum over parts to obtain the quadratic-form inequality.
\end{proof}

\begin{lemma}[Constant uniform-prior revealing]
\label{lem:revealing}
For both candidates and every \(h<H\),
\(\|G_{M_\theta,h}^{-1}\|_1\le35/9\).
\end{lemma}
\begin{proof}
Fix \(2\le h<H\), and order the states as \(\theta,1-\theta\).
Let \(E_h\) be the event that the current action resets to \(1-\theta\).
If \(s_h=1-\theta\), \eqref{eq:statefrommemory} forces
\(m_h=1-\theta\), and this action is preferred. If \(s_h=\theta\),
then \(m_h\in\{\never,\theta\}\), and the action is nonpreferred in both
cases. Consequently,
\[
 q:=\Pp(E_h\mid s_h=\theta)=\frac16,\qquad
 r:=\Pp(E_h\mid s_h=1-\theta)=\frac23 .
\]
The event is part of the future, so its channel is a coarsening of \(U\):
\[
 U_E=\begin{pmatrix}q&r\\1-q&1-r\end{pmatrix}.
\]
Write \(G_E=U_E^\top\diag(U_E\one)^{-1}U_E\).
Both \(G\) and \(G_E\) are symmetric, nonnegative, and stochastic, since
\(G\one=U^\top\one=\one\). In dimension two their eigenvectors are
\(\one\) and \((1,-1)^\top\). Direct calculation gives the nontrivial
eigenvalue of \(G_E\):
\begin{equation}
 \lambda_E=
 \frac{(q-r)^2}{(q+r)(2-q-r)}
 =\frac9{35}.
 \label{eq:eigen}
\end{equation}
Lemma~\ref{lem:coarsen} implies that the nontrivial eigenvalue
\(\lambda\) of \(G\) is at least \(\lambda_E\). Positive semidefiniteness
and stochasticity give \(0<\lambda\le1\), and
\[
 G^{-1}=\frac12
 \begin{pmatrix}1+\lambda^{-1}&1-\lambda^{-1}\\
                 1-\lambda^{-1}&1+\lambda^{-1}\end{pmatrix}.
\]
Its induced 1-norm is \(\lambda^{-1}\le35/9\).
At stage 1 the outcome matrix has one stochastic column and its Gram
matrix is the scalar one.
\end{proof}

The event used to prove the lemma can depend on the candidate under
consideration. The argument establishes the revealing condition separately in each
model; the event is not supplied to the estimator.

\begin{remark}[Terminal stage]
\label{rem:terminal}
At \(h=H\), the outcome is just \(Y\), and
\[
 U_H=\begin{pmatrix}3/4&1/4\\1/4&3/4\end{pmatrix},\qquad
 G_H=\begin{pmatrix}5/8&3/8\\3/8&5/8\end{pmatrix},
 \qquad \|G_H^{-1}\|_1=4.
\]
Thus imposing uniform-prior revealing on every layer, including the
terminal one, changes the common bound to \(C_F=4\).
\end{remark}

\subsection{Prior-weighted outcome revealing}
\begin{lemma}[Constant prior-weighted revealing]
\label{lem:weighted}
For both candidates and every \(h<H\),
\(\|K_{M_\theta,h}^{-1}\|_1\le9\).
\end{lemma}
\begin{proof}
The initial-stage matrix is again the scalar one. For \(2\le h<H\), write
\(\rho=\Pp(s_h=1-\theta)\), so Lemma~\ref{lem:belief} gives
\(\rho\in[1/3,2/3]\). Set
\(X=\one\{s_h=1-\theta\}\) and \(Z=\E[X\mid F_h]\).
The matrix \(K\) is nonnegative and column-stochastic. Its diagonal
entries are \(\E[(1-Z)^2]/(1-\rho)\) and \(\E[Z^2]/\rho\).
One eigenvalue is one; subtracting one from its trace shows that the
other is
\begin{equation}
 \lambda_F=\frac{\Var(Z)}{\rho(1-\rho)}.
 \label{eq:posteriorvariance}
\end{equation}
The law of total variance and the measurability of \(E_h\) with respect
to \(F_h\) imply
\(\Var(Z)\ge\Var(\E[X\mid E_h])\).
For a binary channel with conditional probabilities \(q,r\), let
\(m=(1-\rho)q+\rho r\). Bayes' rule, or the covariance formula for two
binary variables, gives
\[
 \frac{\Var(\E[X\mid E_h])}{\rho(1-\rho)}
 =\frac{\rho(1-\rho)(r-q)^2}{m(1-m)}
 \ge\frac{(2/9)(1/4)}{1/4}=\frac29 .
\]
Write \(K=\left(\begin{smallmatrix}1-a&b\\a&1-b\end{smallmatrix}\right)\).
Then \(a,b\ge0\), \(\lambda_F=1-a-b>0\), and direct inversion yields
\[
 \|K^{-1}\|_1
 =\frac{1+|a-b|}{\lambda_F}
 \le\frac2{\lambda_F}\le9 .
\]
\end{proof}

This establishes the obstruction under both normalizations. Changing
the latent-state prior in the revealing matrix does not remove the
information loss.

\subsection{Information in a logged trajectory}
\begin{lemma}[Only all-hold trajectories distinguish the models]
\label{lem:kl}
Let \(P_\theta\) denote the complete observable-trajectory law under
\(M_\theta,\pb\). Then
\begin{equation}
 \KL(P_0\Vert P_1)=\frac{p_H}{2}\log3,\qquad
 p_H=\frac13(2/3)^{H-2}.
 \label{eq:trajectorykl}
\end{equation}
Both observable-trajectory laws have the same support.
\end{lemma}
\begin{proof}
Let \(\calE_H=\{a_1=\cdots=a_{H-1}=\hold\}\).
The nonterminal action sequence has the same distribution in both
models, because observations are blank and the logger is a common
function of previous actions. Its all-hold probability is \(p_H\).
If any reset occurs, the last reset determines the terminal state in
both candidates. Thus the conditional law of \(Y\) is the same under
both models for every action sequence outside \(\calE_H\).
On \(\calE_H\), it is \(\Ber(1/4)\) under \(M_0\) and
\(\Ber(3/4)\) under \(M_1\).
The final independent uniform action contributes no divergence. The
KL chain rule therefore gives
\[
 \KL(P_0\Vert P_1)
 =p_H\kl(1/4,3/4)=\frac{p_H}{2}\log3.
\]
Every nonterminal action has positive probability, and both terminal
emissions have positive probability at either bit. The supports coincide.
\end{proof}

\begin{proof}[Proof of Theorem~\ref{thm:main}]
The construction establishes claims (i), (ii), and the value separation.
Lemmas~\ref{lem:belief}, \ref{lem:revealing}, and \ref{lem:weighted}
establish (iii).
Given an accurate estimator, define a test \(\widehat\theta\) that
returns one when \(\widehat J>1/2\), and zero otherwise.
An estimate within \(1/8\) of either true value implies the correct
decision, so both testing errors are at most \(\delta\).
Data processing for relative entropy gives
\[
 \KL(P_0^{\otimes n}\Vert P_1^{\otimes n})
 \ge \kl(1-\delta,\delta).
\]
Indeed the event \(\{\widehat\theta=0\}\) has probability at least
\(1-\delta\) under \(M_0\) and at most \(\delta\) under \(M_1\);
binary KL is minimized at those boundary probabilities.
Independence of the trajectories and Lemma~\ref{lem:kl} prove
\eqref{eq:mainlower}.

The estimator's randomization, including any adaptive sequence of
simulations in the supplied candidates, is a common stochastic kernel
applied to the observed data and known inputs. Data processing applies
to that kernel regardless of its computational cost.
\end{proof}

The last step is the standard testing reduction used in information
lower bounds; related change-of-measure arguments are developed, for
example, by \citet{kaufmann2016}. The contribution here is the simultaneous
coverage-preserving construction.

\section{Proofs for the optimal estimator}
\label{app:exactproof}
\begin{proof}[Proof of Proposition~\ref{prop:experiment}]
Equation~\eqref{eq:ternary} follows from the construction.
For \(W=\bot\), the full-trajectory likelihood is identical under
both candidates on every reset trajectory, and the normalizing
probability \(1-p_H\) is also common. Therefore their conditional
trajectory laws coincide. For \(W=0\) or \(W=1\), all nonterminal actions
are hold and the terminal observation is fixed; only the common
independent final action remains to be sampled. These conditional laws
define the reconstruction kernel. Mapping full trajectories to \(W\)
and reconstructing in the reverse direction shows equality of the
optimal risks.
\end{proof}

\begin{proof}[Proof of Theorem~\ref{thm:minimax}]
The likelihood ratio \(dQ_1^{\otimes n}/dQ_0^{\otimes n}\) is
\(3^{S_n}\). The likelihood-ratio test selects model one when \(S_n>0\),
model zero when \(S_n<0\), and uses a fair coin at a tie. It minimizes
equal-prior Bayes error by selecting the larger likelihood at every
data outcome. The experiment is symmetric under exchanging zero and
one, so the two error probabilities of this test agree. Its maximum
error equals its Bayes error, which lower-bounds the maximum error of
any test. This proves minimax optimality and \eqref{eq:exactrisk}.

A value estimate that succeeds within \(1/8\) identifies the model by
the threshold \(1/2\). Conversely, a test can return the corresponding
value \(1/4\) or \(3/4\), succeeding exactly when its decision is correct.
Thus the testing and value-estimation minimax risks coincide.

If all \(n\) symbols are \(\bot\), the data have the same law under both
models and the equal-prior conditional error is \(1/2\). This event has
probability \((1-p_H)^n\), proving the lower bound.
For the upper bound, the Bayes overlap formula and
\(\min\{a,b\}\le\sqrt{ab}\) imply
\[
 e_n=\frac12\sum_w\min\{Q_0^{\otimes n}(w),Q_1^{\otimes n}(w)\}
 \le \frac12\Aff(Q_0,Q_1)^n .
\]
Directly from \eqref{eq:ternary},
\(\Aff(Q_0,Q_1)=1-p_H+(\sqrt3/2)p_H\), proving \eqref{eq:riskbounds}.
\end{proof}

\begin{proof}[Proof of Corollary~\ref{cor:rate}]
The lower bound is Theorem~\ref{thm:main}. The upper bound follows by
making the right side of \eqref{eq:riskbounds} at most \(\delta\).
For \(0<\delta\le1/4\), binary KL is bounded above and below by
universal positive multiples of \(\log(1/\delta)\).
Also \(-\log(1-c_\star p_H)\ge c_\star p_H\).
These observations and \(p_H^{-1}=3(3/2)^{H-2}\) establish the stated
uniform order, including the harmless integer rounding.
\end{proof}

\section{Logger memory and conditional future laws}
\label{app:memory}
\subsection{The same lane can lead to different logged futures}
\begin{proposition}[Failure of belief-based factorization]
\label{prop:factorization}
For either candidate, there exist two positive-probability histories
at the same nonterminal stage with identical physical-state beliefs
but different conditional future laws. In particular, the identity
\begin{equation}
 \Pp^{\pb}(F_h=\cdot\mid\tau_{h-1})
   =U_{M,h}b_{M,h}(\tau_{h-1})
 \label{eq:failedfactor}
\end{equation}
cannot hold for all histories.
\end{proposition}
\begin{proof}
In \(M_0\), at stage \(h=2\), compare the histories whose first actions
are \(\hold\) and \(\resetzero\). Both beliefs equal \(e_0\).
After the first history the next-action probabilities, in the order
\((\hold,\resetzero,\resetone)\), are \((2/3,1/6,1/6)\);
after the second they are \((1/6,2/3,1/6)\).
Their total variation distance is \(1/2\).
Since the current action is a coordinate of \(F_h\), the future laws
also differ. The right side of \eqref{eq:failedfactor} is the same for
the two histories, so it cannot equal both laws.
For \(M_1\), use the histories \(\hold\) and \(\resetone\).
\end{proof}

The logger's actions reveal a correlation with the physical state
under the data-collection law. They do not establish that the physical
state is a sufficient statistic for the logger's continuation. The
prior-weighted normalization changes a marginal distribution, not
this conditional-independence property.

\subsection{Including logger memory makes a rare state explicit}
\begin{proposition}[Coverage after incorporating logger memory]
\label{prop:augmented}
Fix a candidate model and augment the physical state by \(m_h\).
After restricting to its reachable augmented states, the belief
second-moment matrix at stage \(h\ge2\) has an eigenvalue \(w_h\).
Its inverse minimum eigenvalue is at least
\[
 \frac1{w_h}=3(3/2)^{h-2}.
\]
\end{proposition}
\begin{proof}
The three reachable augmented states are
\((\theta,\never),(0,0),(1,1)\).
Each is determined by the observable action prefix in the fixed
model, so its belief vector is a coordinate vector.
The second-moment matrix is diagonal with entries
\(w_h,(1-w_h)/2,(1-w_h)/2\).
\end{proof}

The logger can be represented as memoryless on this augmented process,
but constant physical-state coverage does not imply constant augmented-state
coverage. A common augmented model class retaining unreachable states
would only make the full-rank condition harder to satisfy.

\subsection{Three memory labels do not imply a short observation window}
For any fixed suffix length \(L\), choose a horizon long enough that a
stage remains after that suffix and before the terminal observation.
Histories beginning with reset zero and reset one, followed by \(L\)
holds, share the same last \(L\) action-observation pairs but retain
belief distance two in \(\ell_1\).
Within \(M_0\), all-hold and reset-zero histories followed by the same
hold suffix have identical beliefs but next-action distributions at
\(\ell_1\) distance one.
All these histories have positive logging probability.
Thus the number of logger memory states can be constant while the
window needed to reproduce its continuation grows with the horizon.

This observation delineates the role of forgetting assumptions in
positive results. For example, the future-dependent analysis of
\citet{zhu2026} uses forgetting and additional approximation and
coverage conditions. The present family admits no uniformly short
forgetting window across horizons at the fixed accuracies above.
For any individual finite horizon, keeping the entire history is
still possible; no impossibility of that trivial window is claimed.

\section{A compact outcome matrix for exact calculations}
\label{app:compact}
The main proof uses only a binary coarsening of each future.
Here we give a statistic retaining the full information about the
current physical state within a fixed candidate model.
This is distinct from the three-symbol statistic in
Proposition~\ref{prop:experiment}, whose unknown is the model index.

Fix \(M_\theta\), a stage \(2\le h\le H\), and \(L=H-h\).
Compress the future to the time and value of its first reset,
\((t,r)\) with \(t\in\{0,\ldots,L-1\}\) and \(r\in\{0,1\}\),
if a reset occurs. If no reset occurs, record \((\varnothing,y)\).
There are \(2L+2\) possible symbols.

Let \(a_m=\pb(\hold\mid m)\) and
\(b_{m,r}=\pb(\text{reset to }r\mid m)\) be read from the logger table.
Put \(w=w_h\). The conditional distribution of memory given the
physical state is
\begin{align}
 \Pp(m_h=\never\mid s_h=\theta)&=\frac{2w}{1+w},&
 \Pp(m_h=\theta\mid s_h=\theta)&=\frac{1-w}{1+w},\label{eq:conditionalmemory}\\
 \Pp(m_h=1-\theta\mid s_h=1-\theta)&=1.&&\nonumber
\end{align}
The compressed outcome matrix is therefore
\begin{align}
 \bar U_h((t,r),s)
   &=\sum_m\Pp(m_h=m\mid s_h=s)a_m^t b_{m,r},\label{eq:compactreset}\\
 \bar U_h((\varnothing,y),s)
   &=\sum_m\Pp(m_h=m\mid s_h=s)a_m^L O_H(y\mid s).
 \label{eq:compacthold}
\end{align}

\begin{proposition}[Exact preservation of both revealing matrices]
\label{prop:compact}
If \(T(F_h)\) is the compressed statistic above, then there is a
reconstruction kernel \(V_h\), independent of \(s_h\), such that
\[
 U_h(f,s)=V_h(f\mid T(f))\bar U_h(T(f),s).
\]
Both the uniform-prior and prior-weighted Gram matrices computed
from \(\bar U_h\) equal their full-future counterparts.
\end{proposition}
\begin{proof}
After the first future reset to \(r\), the physical bit and logger
memory both equal \(r\). The law of all subsequent actions and
observations is independent of the bit at stage \(h\), conditional
on the reset time and value. Before this first reset all actions
are hold and all observations are fixed.
If no reset occurs, the compressed terminal symbol determines the
entire future. These facts give the reconstruction kernel, with
\(\sum_{f:T(f)=t}V_h(f\mid t)=1\) for each category.

For a category \(t\), write \(v_f=V_h(f\mid t)\) and
\(u_t=\bar U_h(t,\cdot)\).
The contribution of that category to the full uniform-prior Gram
matrix is
\[
 \sum_{f:T(f)=t}
 \frac{(v_fu_t)^\top(v_fu_t)}{v_fu_t\one}
 =\frac{u_t^\top u_t}{u_t\one}\sum_{f:T(f)=t}v_f
 =\frac{u_t^\top u_t}{u_t\one}.
\]
The same calculation with denominator \(v_fu_tp_h\), followed by
left multiplication by \(\diag(p_h)\), proves the prior-weighted
identity. Terms with \(v_f=0\) are omitted.
\end{proof}

The reconstruction is within a fixed model; it may depend on that
model. This is sufficient for computing revealing matrices and is
not a claim that model identity is revealed by the compressed
state experiment.

\section{A deterministic-terminal variant}
\label{app:deterministic}
Replace \eqref{eq:emission} by \(Y=s_H\), leaving transitions and
policies unchanged. The target values become zero and one.
The belief and event-based revealing proofs are unchanged.
The three-symbol experiment becomes
\[
 Q_0=(1-p_H,p_H,0),\qquad Q_1=(1-p_H,0,p_H).
\]
One non-erased observation identifies the model. Conditional on all
erasures, the candidates remain indistinguishable. Consequently,
\[
 \TV(P_0^{\otimes n},P_1^{\otimes n})=1-(1-p_H)^n,
 \qquad e_n=\frac12(1-p_H)^n .
\]
Here the non-erased observations are singular across the two models.
The noisy version in the main theorem has common observable support
and finite KL, so its lower bound does not rely on singularity.

\section{Experimental details and additional diagnostics}
\label{app:experiment_details}
\subsection{Experimental protocol}
Each simulated trajectory records $(h,\text{corridor})$ at nonterminal
stages and $(H,Y)$ at termination. The estimator receives no latent-state
information. Figure~\ref{fig:mechanism} exposes the two lanes solely
to explain the environment. The lane-observation control
adds the lane to the stage-2 observation. Every gate transition counts
as one step in this directed grid; no uncounted movement substeps
are used.

The behavior policy uses only its past actions. A uniform six-sided
draw chooses its preferred action on four faces and each alternative
on one face. The first and final logging actions are uniform over
three actions. The target always continues. The estimator retains
counts $(N_-,N_+)$ of no-reset episodes with terminal reward zero
and one, selects a candidate from the sign of $N_+-N_-$, and draws
an independent fair tie-breaker. For the lane-observation control,
a first continue gives a noiseless signed observation of the model;
in the absence of one, the estimator uses a fair coin.

The master seed is 20260914. NumPy SeedSequence spawns separate
streams by protocol, horizon, candidate, and replicate block.
Each protocol/horizon/candidate uses 256 independent datasets.
The offline budget set is specified in Section~\ref{sec:experiments};
controls use $n\in\{1,2,4,8,16,32,64,128\}$ at
$H\in\{4,12,24\}$. Prefix reuse avoids resimulating smaller budgets
and creates correlation between plotted budgets, not between
replicates at a fixed budget. The full sweep generates 172,233,728
episodes and 3,788,231,680 physical transitions in vectorized batches.
The recorded run uses Python 3.12.14 and NumPy 2.3.5; figures use
Matplotlib 3.11.2. No policy or value function is trained. Population
coverage matrices are computed in exact rational arithmetic before
decimal reporting, with norms maximized over candidates and stages $h<H$.

\subsection{Computing the exact risk and confidence intervals}
Let $M$ count the informative episodes. Then
$M\sim\operatorname{Bin}(n,p_H)$. Under $M_0$, the number of positive
terminal bits conditional on $M=m$ is
$B_m\sim\operatorname{Bin}(m,1/4)$. Consequently,
\begin{equation}
 e_n=\sum_{m=0}^n\binom nm p_H^m(1-p_H)^{n-m}
 \left[\Pp(B_m>m/2)+\frac12\Pp(B_m=m/2)\right].
 \label{eq:binomialrisk}
\end{equation}
This is another expression for Theorem~\ref{thm:minimax}, not an
approximation to it. We evaluate the outer binomial probabilities recursively and the
bracketed term using integer binomial coefficients. At large $n$, a tail is omitted only when its
Chernoff probability bound is below $10^{-14}$. We use no Poisson
approximation. Binary search in $n$ returns the first integer whose
computed error is at most 0.1.

For the target-policy data control, $p_H=1$ and every episode
contributes a noisy bit. Hoeffding's inequality gives $e_n\le e^{-n/8}$,
so $8\log(1/\delta)$ episodes, rounded upward, suffice. The number
of episodes is independent of $H$, although each episode costs $H$
primitive decisions.

For $k$ failures out of $R$ independent datasets, let
$\widehat q=k/R$ and $z=1.95996398454$. The Wilson interval is
\[
 \frac{\widehat q+z^2/(2R)
       \ \pm\ z\sqrt{\widehat q(1-\widehat q)/R+z^2/(4R^2)}}
      {1+z^2/R}.
\]
The main figures use $R=512$, pooling the two symmetric testing
problems. Baseline estimators can have different errors in the two
models, so their additional comparison uses $M_0$ alone and $R=256$.

\begin{figure}[!htbp]
\centering
\includegraphics[width=0.92\linewidth]{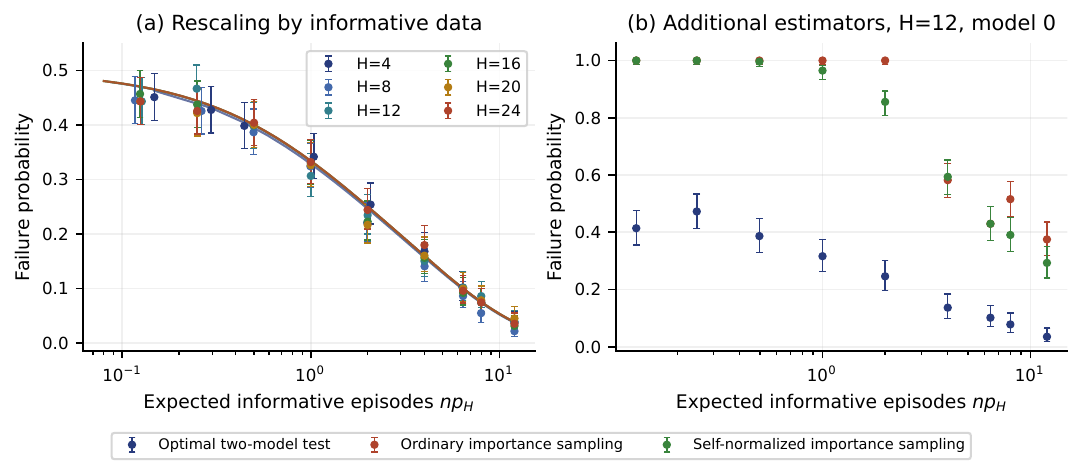}
\caption{Additional diagnostics. (a) Exact curves and simulations
against expected informative sample count; intervals use 512 datasets.
(b) Failure at accuracy $1/8$ for three estimators on the same datasets
at $H=12$ in $M_0$; intervals use 256 datasets. All intervals are
pointwise 95\% Wilson intervals.}
\label{fig:additional}
\end{figure}

\subsection{Rescaling and importance-sampling comparisons}
Figure~\ref{fig:additional}(a) replots the error against $np_H$,
the expected number of informative episodes. The near-alignment
illustrates the effective amount of data. It is not an exact
finite-sample collapse: the distribution of
$\operatorname{Bin}(n,p_H)$ depends on both parameters.

For a secondary comparison, let $I_i$ indicate a no-reset episode.
After marginalizing the irrelevant final action, its trajectory
importance weight is $I_i/p_H$. We compute ordinary importance
sampling and its self-normalized version:
\[
 \widehat J_{\mathrm{IS}}=\frac{\sum_i I_iY_i}{np_H},
 \qquad
 \widehat J_{\mathrm{WIS}}=
 \begin{cases}
   \sum_i I_iY_i/\sum_i I_i,&\sum_i I_i>0,\\
   1/2,&\sum_i I_i=0.
 \end{cases}
\]
The ordinary estimate is not clipped. The self-normalized estimate
uses the midpoint when its denominator is zero; at accuracy $1/8$
that choice fails in both candidates. Unlike the optimal test, these
estimators do not exploit the restriction of the target value to
$\{1/4,3/4\}$. Figure~\ref{fig:additional}(b) is therefore a
comparison of specified estimators on this family, not a general
algorithmic superiority claim.
\end{document}

%% file: threshold_table.tex
\begin{tabular}{rrrrr}
\toprule
$H$ & $N_H(0.1)$ & Analytic failure & Observed failure & 95\% interval\\
\midrule
4 & 43 & 9.89\% & 10.2\% & [7.8, 13.1]\% \\
8 & 219 & 9.97\% & 8.6\% & [6.5, 11.3]\% \\
12 & 1,110 & 9.99\% & 10.2\% & [7.8, 13.1]\% \\
16 & 5,618 & 10.00\% & 9.2\% & [7.0, 12.0]\% \\
20 & 28,441 & 10.00\% & 10.0\% & [7.7, 12.9]\% \\
24 & 143,982 & 10.00\% & 9.6\% & [7.3, 12.4]\% \\
\bottomrule
\end{tabular}

%% file: main.bbl
\begin{thebibliography}{5}
\providecommand{\natexlab}[1]{#1}
\providecommand{\url}[1]{\texttt{#1}}
\expandafter\ifx\csname urlstyle\endcsname\relax
  \providecommand{\doi}[1]{doi: #1}\else
  \providecommand{\doi}{doi: \begingroup \urlstyle{rm}\Url}\fi

\bibitem[Kaufmann et~al.(2016)Kaufmann, Capp{\'e}, and Garivier]{kaufmann2016}
Emilie Kaufmann, Olivier Capp{\'e}, and Aur{\'e}lien Garivier.
\newblock On the complexity of best-arm identification in multi-armed bandit
  models.
\newblock \emph{Journal of Machine Learning Research}, 17\penalty0
  (1):\penalty0 1--42, 2016.
\newblock URL \url{https://jmlr.org/papers/v17/kaufman16a.html}.

\bibitem[Uehara et~al.(2023)Uehara, Kiyohara, Bennett, Chernozhukov, Jiang,
  Kallus, Shi, and Sun]{uehara2023}
Masatoshi Uehara, Haruka Kiyohara, Andrew Bennett, Victor Chernozhukov, Nan
  Jiang, Nathan Kallus, Chengchun Shi, and Wen Sun.
\newblock Future-dependent value-based off-policy evaluation in {POMDPs}.
\newblock In \emph{Advances in Neural Information Processing Systems},
  volume~36, 2023.
\newblock URL \url{https://arxiv.org/abs/2207.13081}.

\bibitem[Zhang and Jiang(2024)]{zhang2024}
Yuheng Zhang and Nan Jiang.
\newblock On the curses of future and history in future-dependent value
  functions for off-policy evaluation.
\newblock In \emph{Advances in Neural Information Processing Systems},
  volume~37, 2024.
\newblock URL \url{https://arxiv.org/abs/2402.14703}.

\bibitem[Zhang and Jiang(2025)]{zhang2025}
Yuheng Zhang and Nan Jiang.
\newblock Statistical tractability of off-policy evaluation of
  history-dependent policies in {POMDPs}.
\newblock In \emph{International Conference on Learning Representations}, 2025.
\newblock URL \url{https://arxiv.org/abs/2503.01134}.

\bibitem[Zhu and Lu(2026)]{zhu2026}
Youheng Zhu and Yiping Lu.
\newblock A covering framework for offline {POMDPs} learning using belief space
  metric.
\newblock \emph{arXiv preprint arXiv:2603.03191}, 2026.
\newblock URL \url{https://arxiv.org/abs/2603.03191}.

\end{thebibliography}
